\documentclass{article}

    \usepackage[preprint,nonatbib]{neurips_2020}

\usepackage[utf8]{inputenc} 
\usepackage[T1]{fontenc}    
\usepackage{hyperref}       
\usepackage{url}            
\usepackage{booktabs}       
\usepackage{amsfonts}       
\usepackage{nicefrac}       
\usepackage{microtype}      

\usepackage{times}
\usepackage{amsmath,amsthm,amssymb,stmaryrd,pifont,wasysym,algorithm,algorithmic}
\usepackage{amssymb,bbm,centernot,url}
\usepackage{tabularx}
\usepackage{caption}
\usepackage{wrapfig}
\usepackage{graphicx} 

\theoremstyle{plain}
\newtheorem{theo}{Theorem}

\newcommand{\CI}{\mathrel{\perp\mspace{-10mu}\perp}}
\newcommand{\nCI}{\centernot{\CI}}

\usepackage{tikz}
\tikzset{
  treenode/.style = {shape=rectangle, rounded corners,
                     draw, align=center,
                     top color=white, bottom color=blue!20},
  root/.style     = {treenode, font=\Large, bottom color=red!30},
  env/.style      = {treenode, font=\ttfamily\normalsize},
  dummy/.style    = {circle,draw}
}

\usetikzlibrary{positioning}
\newdimen\nodeDist
\title{Latent Instrumental Variables as Priors in Causal Inference based on Independence of Cause and Mechanism}

\author{%
 Nataliya Sokolovska \\
NutriOmics,  Sorbonne University\\
Paris,  France\\
  \texttt{nataliya.sokolovska@sorbonne-universite.fr} \\
  \AND 
  Pierre-Henri Wuillemin\\
  LIP6, Sorbonne University \\
  Paris, France \\
  \texttt{pierre-henri.wuillemin@lip6.fr}
  }

\begin{document}

\maketitle

\begin{abstract}

Causal inference methods based on conditional independence construct Markov equivalent graphs, and cannot be applied to bivariate cases. The approaches based on independence of cause and mechanism state, on the contrary, that causal discovery can be inferred for two observations. In our contribution, we challenge to \textit{reconcile} these two research directions. We study the role of latent variables such as latent instrumental variables and hidden common causes in the causal graphical structures. We show that the methods based on the independence of cause and mechanism, \textit{indirectly contain traces of the existence of the hidden instrumental variables}. We derive a novel algorithm to infer causal relationships between two variables, and we validate the proposed method on simulated data and on a benchmark of cause-effect pairs. We illustrate by our experiments that the proposed approach is simple and extremely competitive in terms of empirical accuracy compared to the state-of-the-art methods.  
\end{abstract}

\section{Introduction}

Causal inference purely from non-temporal observational data is challenging. Instead of learning causal structure of an entire dataset, some researchers focus on analysis of causal relations between two variables only. The state-of-the-art conditional independence-based causal discovery methods (see, e.g., \cite{PearlCausality,PC}) construct graphs that are Markov equivalent, but these methods are not applicable in the case of two variables, since $X \rightarrow Y$ and $Y \rightarrow X$ are Markov equivalent.

The statistical and probabilistic causal inference methods based on assumptions of independence of cause and mechanism (see~\cite{Peters17} for general overview) appeared relatively recently, and achieve very reasonable empirical results. The main idea behind these methods is as follows: if a simple function that fits data exists, then it is likely that it also describes a causal relation in the data. 

\textit{The main goal of our paper is to try to reconcile these two modern viewpoints on causal inference}: the research direction initiated by~\cite{PC,PearlCausality} which is based on the assumption of conditional independencies, and the more recent research avenue where the main claim is that causal inference between two observations only is feasible~\cite{Scholkopf13,Parascandolo18,Daniusis10,Scholkopf12,SgouritsaCURE,Liu16} and whose theory relies on the  independence of cause and mechanism. 

To illustrate the intuition behind our approach, let us consider an example from~\cite{Peters17} with altitude and temperature, where $A$ is altitude, $T$ is temperature, $P(A)$ are cities locations, $P(T|A)$ is the physical mechanism of temperature given altitude, and it can be shown that changing the cities locations $P(A)$, the conditional probability $P(T|A)$  does not change. The postulate of independence of cause and mechanism allows to infer the causal direction $A \rightarrow T$. Any latent variables are ignored in this case. However, the cities locations depend on a country, since each country has different urban policy, population density, etc. Thus, in this example $P(A)$ has at least one latent variable which is county $C$. However, no matter what country is chosen, the physical mechanism $P(T|A)$ holds, and the true underlying causal structure is $C \rightarrow A \rightarrow T$. A country defines distribution of cities. Having two or more countries leads to a family of distributions. This \textit{mixture of probability distributions} is \textit{independent} from $P(T|A)$. Thus, this example also explains what is meant under the \textit{independence between probability distributions}.

Our contribution is multi-fold:
\begin{itemize}
\item Our main theoretical result is an alternative viewpoint on the recently appeared causal inference algorithms that are based on the independence of cause and mechanism\footnote{Here, we follow the simplification used by~\cite{Peters17}, however, we are aware that the independence of our interest is between the \textit{prior} of the cause and the mechanism.}. 
\item Assuming the existence of the hidden instrumental variables, we propose a novel method of causal inference. Since we consider a bivariate causal inference case where only $X$ and $Y$ are observed, we also propose an approach to estimate the latent instrumental variables for cases where the cluster assumption for $X$ and $Y$ holds.
\item We investigate the role of latent common causes, and we propose a method to identify them.
\item We validate our method on a synthetic data set on which we perform extensive numerical experiments, and on the Cause-Effect benchmark which is actively used by the causal inference community. 
\end{itemize}

The paper is organized as follows. Section~\ref{relatedWorks} discusses the state-of-the-art methods of bivariate causal inference. Preliminaries on the instrumental variables are provided in Section~\ref{sec:prelim}. We consider the role of the instrumental variables for causal inference, and we introduce our approach in Section~\ref{sec:ourmethod}.  
In Section~\ref{sec:exp}, we discuss the results of our numerical experiments on synthetic, as well as on standard challenges. Concluding remarks and perspectives close the paper.

\section{Related Work}
\label{relatedWorks}

In this section, we discuss the state-of-the-art methods of bivariate causal inference and the corresponding assumptions. 
In the current work, we focus on a family of causal inference methods which are based on a postulate telling that if $X \rightarrow Y$, then the marginal distribution $P(X)$ and the conditional distribution $P(Y|X)$ are independent \cite{Janzing10,Janzing12,SgouritsaCURE}. These approaches provide causal directions based on the estimated conditional and marginal distributions from observed non-temporal data.
Probably one of the oldest and well-studied models describing causal relations that are necessary to mention, are Structural Causal Models (SCM). A SCM where $X \rightarrow Y$ is defined as follows:
\begin{align}
X &= N_X, \hspace{0.5cm} Y = f_Y(X, N_Y),
\end{align} 
where $N_X$ and $N_Y$ are independent. Given $f_Y$ and the noise distributions $P_{N_Y}$ and $P_{N_X}$, we can sample data following a SCM.

A recently proposed but already often used \textit{Postulate of Independence of Cause and Mechanism} is formulated as follows (see, e.g., \cite{Janzing10,Janzing12,SgouritsaCURE}).  If $X$ causes $Y$, then $P(X)$ and $P(Y|X)$ estimated from observational data contain no information about each other. Looking for a parallel between the postulate and the SCM, we assume that in a SCM,  $f_Y$ and $P_{N_Y}$ contain no information about $P_X$ and visa versa. The postulate describes the independence of mechanisms, and tells that a causal direction can be inferred from estimated marginal and conditional probabilities of random variables from a data set. In the following, we investigate this research direction. 

It is not obvious how to formalise the independence of the marginal and conditional probabilities. A reasonable claim~\cite{Peters17} is that an optimal measure of dependence is the algorithmic mutual information that relies on the description length in the sense of Kolmogorov complexity. Since the exact computations are not feasible, there is a need for a practical and reliable approximation encoding that in a causal direction  $P(X)$ and $P(Y|X)$ would require more compact model classes than in an anticausal direction. 

Two families of methods of causal inference dealing with bivariate relations are often discussed. For a more general overview of causal structure learning see \cite{HeinzeDeml17,Peters17}. 
 Additive noise models (ANM) introduced by \cite{Hoyer09} and \cite{Peters14} is an attempt to describe causal relations between two variables. The ANM assume that if there is a function $f$ and some noise $E$ such that $Y = f(X) + E$, where $E$ and $X$ are independent, then the direction is inferred to $X \rightarrow Y$.  A generalised extension of the ANM, called post-nonlinear models, was introduced by \cite{Zhang09}. However, the known drawback of the ANM  
is that the model is not always suitable for inference on discrete tasks~\cite{Buhlmann14}. 

Another research avenue exploiting the asymmetry between cause and effect are the linear trace (LTr) method  \cite{Zscheischler11} and information-geometric causal inference (IGCI) \cite{Janzing12}. If the true model is $X \rightarrow Y$, and if $P(X)$ is independent from $P(Y|X)$, then the trace condition if fulfilled in the causal direction, and violated in the anticausal one. The IGCI method exploits the fact that the density of the cause and the log slope of the function-transforming cause to effect are uncorrelated. However, for the opposite direction, the density of the effect and the log slope of the inverse of the function are positively correlated. The trace condition is proved under the assumption that the covariance matrix is drawn from a rotation invariant prior~\cite{Janzing10}. The method was generalized for non-linear cases~\cite{Liu17}, and it was shown that the covariance matrix of the mean embedding of the cause in reproducing kernel Hilbert space is free independent with the covariance matrix of the conditional embedding of the effect given cause.
The case of high-dimensional variables is considered in \cite{Janzing2010} and \cite{Zscheischler11}, where the IGCI is considered. Here, the independence between probability distributions is based on the trace condition.  The identifiability via the trace condition is proved \cite{Janzing2010,Peters17} for deterministic relations, and no theory exists for noisy cases which are much more relevant for real-life applications.

Origo \cite{Budhathoki16} is a causal discovery method based on the Kolmogorov complexity. The Minimum Description Length (MDL) principle can be used to approximate the Kolmogorov complexity for real tasks. Namely, from algorithmic information viewpoint, if $X \rightarrow Y$, then the shortest program that computes $Y$ from $X$ will be more simple, or compact, than the shortest program computing $X$ from $Y$.  The obvious weakness of methods based on the Kolmogorov complexity, and also of Origo, is that the MDL only approximates Komolgorov complexity, and involves unknown metric errors that are difficult to control. The empirical performance is highly dependent on a dataset, and Origo was reported to reach the state-of-the-art performance on the multivariate benchmarks (Acute inflammation, ICDM abstracts, Adult data set), however, it performs less accurately than the ANM on the univariate benchmark of cause-effect pairs with known ground truth (the T\"ubingen data set)~\cite{Mooij16}. We also use this benchmark for our experiments. 

\section{Independence of Probability Distributions and Instrumental Variables}
\label{sec:prelim}

Let $X$ and $Y$ be two correlated variables. In the settings considered by~\cite{Peters17}, in order to decide whether $X\rightarrow Y$ or $Y\rightarrow X$, it is proposed to check if the distributions $P(X)$ and $P(Y|X)$ are independent. As far as we know, this independence between distributions (and not between random variables) does not have any formal definition. However, some useful properties can be derived, and various criteria were constructed for different cases~\cite{Scholkopf13,Parascandolo18,Daniusis10,Scholkopf12,SgouritsaCURE,Liu16}. In this paper, we adopt the following definition. Let $P(X,Y)$ be the joint distribution of $X,Y$ in a population $\cal P$; let $Q(X,Y)$ be the joint distribution of $X,Y$  in another population $\cal Q$. If $X$ is the cause of $Y$, the causal mechanism should be the same in the two distributions:
$$P(X,Y)=P(X)\cdot P(Y|X) \text{ and } Q(X,Y)=Q(X)\cdot P(Y|X),$$
i.e. $P(Y|X)=Q(Y|X)$, on the contrary, $P(X|Y)\neq Q(X|Y)$. More generally, for all mixed populations between $\cal P$ and $\cal Q$, and then, for all mixtures $Q_\lambda=\lambda P + (1-\lambda) Q$ with $\lambda \in [0,1]$, 
$$\forall \lambda \in [0,1], Q_\lambda(X) \CI Q_\lambda(Y|X) \iff Q_\lambda(Y|X)=P(Y|X).$$

Now we consider $\lambda$ as a hyper-parameter for a (latent) prior $I_X$ that allows to select the population ($P(X|I_X=0)=P(X)$, $P(X|I_X=1)=Q(X)$). In this meta-model, $I_X$ and $X$ are dependent, $X$ and $Y$ are dependent. But $I_X$ and $Y$ are independent conditionally to $X$. On the contrary, if we consider $\lambda$ as a hyper-parameter for a (latent) prior $I_Y$, that allows to select the population ($P(Y|I_Y=0)=P(Y)$, $P(Y|I_Y=1)=Q(Y)$). In this meta-model, $I_Y$ and $Y$ are dependent, $X$ and $Y$ are dependent. But since, $P(X|Y)\neq Q(X|Y)$, $I_Y$ and $X$ are not independent, even conditionally to $Y$.

To provide some intuition behind such a mixture model: let $P(X)$ and $Q(X)$ be the distributions of cities locations in two different countries, and $P(Y|X)$ be a physical mechanism predicting weather given location. Then $\lambda$ is the hyper-parameter controlling the proportion of observations in each country, and note that $\lambda$, $P(X)$, and also $Q(X)$ are independent from $P(Y|X)$. 

Such a representation of the problem as a mixture model with latent priors, motivates our proposition to use models with instrumental latent variables.

The aim of models with \textit{instrumental variables}~\cite{Wright28,Heckman97,Angrist01} where $X$, $Y$, and $I_X$ are observed, and $U$ is an unobserved confounder, is to identify the causal effect of $X$ on $Y$. Assuming that the relationships are linear, and applying a linear Gaussian structural causal model, one can write:
\begin{align}
X & = \alpha_0 + \alpha I_X + \delta U + \epsilon_X, \hspace{0.5cm} Y  = \beta_0 + \beta X + \gamma U + \epsilon_Y, 
\end{align} 
where $\epsilon_X$ and $\epsilon_Y$ are noise terms, independent of each other. It is assumed, without loss of generality that $U$, $\epsilon_X$, and $\epsilon_Y$ have mean zero. \textit{Note that the common cause $U$ can be absent}, and we are not going to assume that $U$ exists when modelling dependencies between $X$ and $Y$.
The instrumental variable $I_X$ is uncorrelated with ancestors of $X$ and $Y.$ The instrumental variable is a source of variation for $X$, and it only influences $Y$ through $X$. Studying how $X$ and $Y$ respond to perturbations of $I_X$ can help to deduce how $X$ influences $Y$. A two-stage least squares~\cite{Sawa12} can be used to solve the problem.

\textbf{Probability Distributions as Random Variables.} Similarly to~\cite{Janzing2010,Peters17}, we consider  probability distributions as random variables. $P(X)$ is a function of $X \in [0,1]$, and, thus, they are random variables distributed in $[0,1]$. Note that a model where a probability is randomly generated is an example of a hierarchical model, or of a model with priors, where some parameters are treated as random variables.

\section{Latent Instrumental Variables for Causal Discovery}
\label{sec:ourmethod}

In this section, we show that the methods based on the independence of cause and mechanism, introduced by~\cite{Scholkopf13,Parascandolo18,Daniusis10,Scholkopf12,SgouritsaCURE,Liu16}, also \textit{indirectly contain traces of the existence of the hidden instrumental variable}. It can be seen as follows. $P(X)$ generates $X$ in the approaches proposed and investigated by the scientists mentioned above. In our method, we assume that $X$ are generated by $I_X$. Therefore, there is a strong parallel between $P(X)$ and $I_X$ which are both \textit{priors for the observations}. So, our method also provides some intuition and interpretation of the recently proposed algorithms based on the independence between the ``cause and the mechanism''. We provide some theoretical results on the independence of the causal mechanisms in terms of probability distributions and information theory. These results allow us to derive a novel algorithm of causal inference which is presented in this section below.

Our observations are $X$ and $Y$, two one-dimensional vectors of the same length, and these variables are correlated. Here, we suppose that either causality between these variables exists, and either $X \rightarrow Y$, or $Y \rightarrow X$, or a common latent cause $X \leftarrow U \rightarrow Y$ can be identified, where $U$ is a hidden variable that can impact $X$ and/or $Y$. Let $I_X$ and $I_Y$ denote latent instrumental variables of $X$ and $Y$ respectively. In the current contribution, \textit{we do not observe the instrumental variables, we assume that they exist}, and can be approximated. \textit{We do not assume that $U$ exists, however, we show how its existence can be deduced,} if it is the case. 

\textbf{Assumption 1.}
In case of observational non-temporal data, if $I_X$ exists such that $I_X \rightarrow X$, and if $I_Y$ exists such that $I_Y \rightarrow Y$, and if the random variables $X$ and $Y$ are correlated, then we assume that it is impossible that both $I_X \CI Y | X$ and $I_Y \CI X | Y$ hold.

\begin{theo} 
Let $X$ and $Y$ be two correlated random variables, and they do not have any common cause. We assume that either $X$ causes $Y$ or visa versa. If there exists a random variable $I_X$ such that $I_X \rightarrow X$, and if  $I_X \CI Y | X$, then we are able to infer causality and decide that $X \rightarrow Y$.
\end{theo} 

\begin{proof} 
Several directed acyclic graphs (DAGs) may be Markov equivalent~\cite{PearlCausality,PC}. We assume that once an essential graph is found, the directed arcs of this graph are interpreted causally. Then, the only DAG which is interpreted causally, is the causal graph.

Under the assumption that $I_X \rightarrow X$, and if $I_X \CI Y | X$, then the only possible directed graph is $I_X \rightarrow X \rightarrow Y$ . In case where $I_X \nCI Y | X$, we get $I_X \rightarrow X \leftarrow Y$.

\end{proof}

\begin{theo} 
If the true causal structure is $I_X \rightarrow X \rightarrow Y$, and $X$ and $Y$ do not have any common cause, then $P(Y|X)$ does not contain any information about $P(X)$ and vice versa, however, $P(X|Y)$ and $P(Y)$ are not independent. 
\end{theo} 

\begin{proof}
Assume that $I_X \CI Y | X$. Let us consider the relation between $P(Y|X)$ and $P(X)$. In the following, we treat $P(Y|X)$, $P(X|Y)$, $P(X)$, and $P(Y)$ as random variables. 
We can write:
\begin{align}
P(Y|I_X, X) = P(Y|X).
\label{eq:eq1}
\end{align}
Note that we do not have $P(X)$ in eq.~(\ref{eq:eq1}) when we express $P(Y|X)$ for $I_X \rightarrow X \rightarrow Y$. 
Let us consider the relation between $P(X|Y)$ and $P(Y)$ for the same graphical structure. We get:
\begin{align}
P(X|Y) = \frac{P(Y|X)P(X|I_X)}{P(Y)},
\label{eq:eq2}
\end{align}
where the form of the nominator is due to the fixed dependencies $I_X \CI Y | X$. From eq.~(\ref{eq:eq2}), we clearly see that $P(X|Y)$ is not independent from $P(Y)$ for this graphical structure. 
\end{proof}

\subsection{Construction of the Instrumental Variables}
\label{sec:instrumental}
\textbf{Assumption 2.} (Cluster assumption.) For $X$ and $Y$, the cluster assumption holds: data belonging to the same cluster share similar properties. 

In some tasks the instrumental variables (IV) are observed, and their application is straightforward. In a number of applications, they are not provided. Here, we discuss how they can be approximated, and we draft a procedure to estimate the instrumental variables. In our experiments, in Section~\ref{sec:exp} we apply the proposed method for the IV construction. Note that the identification and characterisation of latent variables is a challenge in itself. Our work is slightly similar to~\cite{Janzing11,Sgouritsa13} in that we apply clustering methods to create the latent variables. Taking into account that only $X$ and $Y$ are observed, the instrumental variables can be constructed using either $X$, $Y$, or both, and an optimal choice of the variables that influence an IV is related to a graphical structure we try to identify and to orient. So, for a structure $I_X \rightarrow X \rightarrow Y$, $I_X$ does not contain information about $Y$, and $I_X$ has to be constructed from $X$ only. On the contrary, in case of $X \rightarrow Y \leftarrow I_Y$, $I_Y$ is not independent from $X$, and $I_Y$ has to contain information about both $X$ and $Y$. 

We propose that the decision whether $I_X$ and $I_Y$ are to be constructed from one or two variables is taken by Algorithm~\ref{alg:Instruments}. The proposed algorithm constructs the instrumental variables separately from $X$, $Y$ ( $I_{X_X}$,  $I_{Y_Y}$), and from both ($I_{X_{XY}}$, $I_{Y_{YX}}$), and tests which instrumental variables are more relevant. We rely on clustering methods for the instrumental variables estimation, in our experiments, we apply the k-means clustering approach, however, other clustering approaches can be used. Algorithm~\ref{alg:approx} drafts the procedures to approximate the candidates for the IV. 

Algorithm~\ref{alg:Instruments} compares the distance (we considered the Euclidean distance in our experiments but another measure, e.g., the Kullback-Leibler can be used) between $I_{X_X}$ and $I_{X_{XY}}$, and between $I_{Y_Y}$ and $I_{Y_{YX}}$. 
The intuition behind this criterion is as follows. If $Y$ influences clustering of $X$ less, than $X$ impacts clustering of $Y$ (the condition $ if (\text{dist}(I_{X_X},I_{X_XY}) < \text{dist}(I_{Y_Y},I_{Y_YX}))$ in the algorithm), then we apply $I_X$ constructed from $X$ only, and $I_Y$ is constructed from $X$ and $Y$. And visa versa. An important remark is that this criterion has a lot in common with the causal discovery methods based on the \textit{Kolmogorov complexity} and the \textit{MDL}: to infer causality, our criterion choses a simpler model.

\subsection{A Symmetric Causal Inference Algorithm}
Once the instrumental variables are fixed, e.g., as it is proposed in the previous section, we apply a symmetric algorithm sketched as a decision tree on Figure~\ref{algo:decision} based on the conditional (in)dependence tests. It takes $I_X$, $I_Y$, $X$, and $Y$, and returns a causal direction. Precisely, if a conditional independence test tells that $Y \CI I_X | X$ is true, then $X \rightarrow Y$ inferred, otherwise we test whether $X \CI I_Y|Y$, and if it is true, then $Y$ causes $X$. The last case where $X$ and $Y$ are correlated but both $Y \CI I_X | X$ and $X \CI I_Y|Y$ are false, let us conclude that there is a common hidden cause $ Y \leftarrow U \rightarrow X$.

\section{Experiments}
\label{sec:exp}

In this section, we illustrate the predictive efficiency of the proposed method on both artificial and real data sets.

\begin{minipage}{0.45\linewidth}
\begin{algorithm}[H]
\caption{Construction of IV Candidates}
\label{alg:approx}
\begin{algorithmic}
\STATE \underline{$I_{X_X}$ (IV variable of $X$ from $X$) }
\STATE Fix a number of clusters $K$
\STATE Cluster $\{X_i\}_{i=1}^N$ into $K$ clusters
\FOR {$i = 1:N$}
\STATE $I_{i,X_{X}}$ is the center of the cluster where $X_i$ belongs
\ENDFOR

\medskip

\STATE \underline{$I_{X_{XY}}$ (IV variable of $X$ from $X$ and $Y$) }
\STATE Fix a number of clusters $K$
\STATE Cluster $\{X_i, Y_i\}_{i=1}^N$ into $K$ clusters
\FOR {$i = 1:N$}
\STATE $I_{i,X_{XY}}$ is the 1st coordinate (corresponding to $X$) of the clusters centres where $(X_i, Y_i)$ belongs
\ENDFOR

\medskip
\STATE \underline{$I_{Y_Y}$ (IV variable of $Y$ from $Y$)}
\STATE is constructed similarly to the IV variable of $X$ from~$X$

\medskip
\STATE \underline{$I_{Y_{YX}}$ (IV variable of $Y$ from $X$ and $Y$) }
\STATE is constructed similarly to the IV variable of $X$ from~$(X,Y)$
(Take the 2nd coordinate of the clusters centres)

\end{algorithmic}
\end{algorithm}
\end{minipage}
\hfill
\begin{minipage}{0.5\linewidth}
\begin{algorithm}[H]

\begin{algorithmic}
\smallskip

\STATE \textbf{Input:} Observations $X$ and $Y$, a clustering algorithm
\STATE \textbf{Output:} Instrumental variables $I_X$ and $I_Y$

\medskip

\STATE \textit{// Construct instrumental variables to be tested}
\STATE Construct IV of $X$, $I_{X_X}$ using $X$ only
\STATE Construct IV of $X$, $I_{X_{XY}}$ using $X$ and $Y$
\STATE Construct IV of $Y$,  $I_{Y_Y}$ using $Y$ only
\STATE Construct IV  of $Y$, $I_{Y_{YX}}$ using $X$ and $Y$

\medskip

\STATE \textit{// Take the decision which IV to use}

\medskip

\IF {$(\text{dist}(I_{X_X},I_{X_{XY}}) < \text{dist}(I_{Y_Y},I_{Y_{YX}}))$}

\STATE \textit{// the IV of $X$ is constructed from $X$ only}
\STATE $I_X  = I_{X_X}$
\STATE \textit{// the IV of $Y$ is constructed from both $X$ and~$Y$}
\STATE $I_Y  = I_{Y_{YX}}$
\smallskip
\ELSE
\STATE \textit{// the IV of $Y$ is constructed from  $Y$}
\STATE  $I_Y  = I_{Y_{Y}}$
\STATE \textit{// the IV of $X$ is constructed from $X$ and $Y$}
\STATE $I_X  = I_{X_{XY}}$
\ENDIF
\end{algorithmic}
\caption{Approximation of the Instrumental Variables (IV)  $I_X$ and $I_Y$ from $X$ and $Y$.}
\label{alg:Instruments}
\end{algorithm}
\end{minipage}

\begin{wrapfigure}{R}{0.45\textwidth}
\begin{tikzpicture}[scale=0.6,
    node/.style={%
      draw,
      rectangle,
    },
  ]

    \node [node] (A) {$Y \CI I_X|X$?};
    \path (A) ++(-135:\nodeDist) node [node] (B) {$X \rightarrow Y$};
    \path (A) ++(-45:\nodeDist) node [node] (C) {$X \CI I_Y | Y$?};
    \path (C) ++(-135:\nodeDist) node [node] (D) {$Y \rightarrow X$};
    \path (C) ++(-45:\nodeDist) node [node] (E) {$X \leftarrow U \rightarrow Y$};

    \draw (A) -- (B) node [left,pos=0.25] {yes}(A);
    \draw (A) -- (C) node [right,pos=0.25] {no}(A);
    \draw (C) -- (D) node [left,pos=0.25] {yes}(A);
    \draw (C) -- (E) node [right,pos=0.25] {no}(A);
\end{tikzpicture}
\caption{A Symmetric Causal Inference Algorithm}
\label{algo:decision}
\end{wrapfigure}
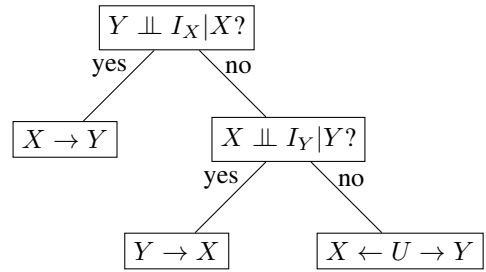

\subsection{Simulated Data}

We consider simple discrete and continuous scenarios. In the discrete case, we fix the structures and the probability distributions on the graphs, and generate binary variables. In the continuous case, we use a Gaussian distribution. We generate the instrumental variables $I_X$ and $I_Y$, $X$ and $Y$, and the hidden variable $U$. We use the \texttt{bnlearn} R package to construct the synthetic data sets, we also use the conditional independence tests from the same package. For our discrete setting with binary variables, we apply an asymptotic mutual information independence test \texttt{ci.test(test='mi')}, and for the continuous setting with Gaussian variables, the exact t-test for Pearson's correlation \texttt{ci.test(test='cor')}. Note that the above mentioned conditional independence tests from the \texttt{bnlearn} R package return ``big'' p-values if variables are conditionally independent, and the p-values are small (with an arbitrary threshold 0.05) for dependent variables. 

We consider and simulate discrete and continuous data for two following scenarios: 1) $X \rightarrow Y$, and $X \leftarrow U \rightarrow Y$. 
We test a various number of observations, from 10 to 10,000, and we observe that in the discrete case, even for such a simple problem as one with variables taking binary values, a big number of observations is needed to obtain a reasonable performance. 
Figure~\ref{fig:simulated1} illustrates the p-values of the conditional independence tests for the discrete (two left plots) and continuous (two right plots) settings. We show the results for both cases  $X \CI I_Y | Y$ and $Y \CI I_X | X$. We observe that for the ground truth $X \rightarrow Y$, $X \CI I_Y | Y$ asymptotically converges to small p-values (close to 0), and $Y \CI I_X | X$ returns big p-values even for big number of observations.
\begin{figure*}
\includegraphics[scale=0.2]{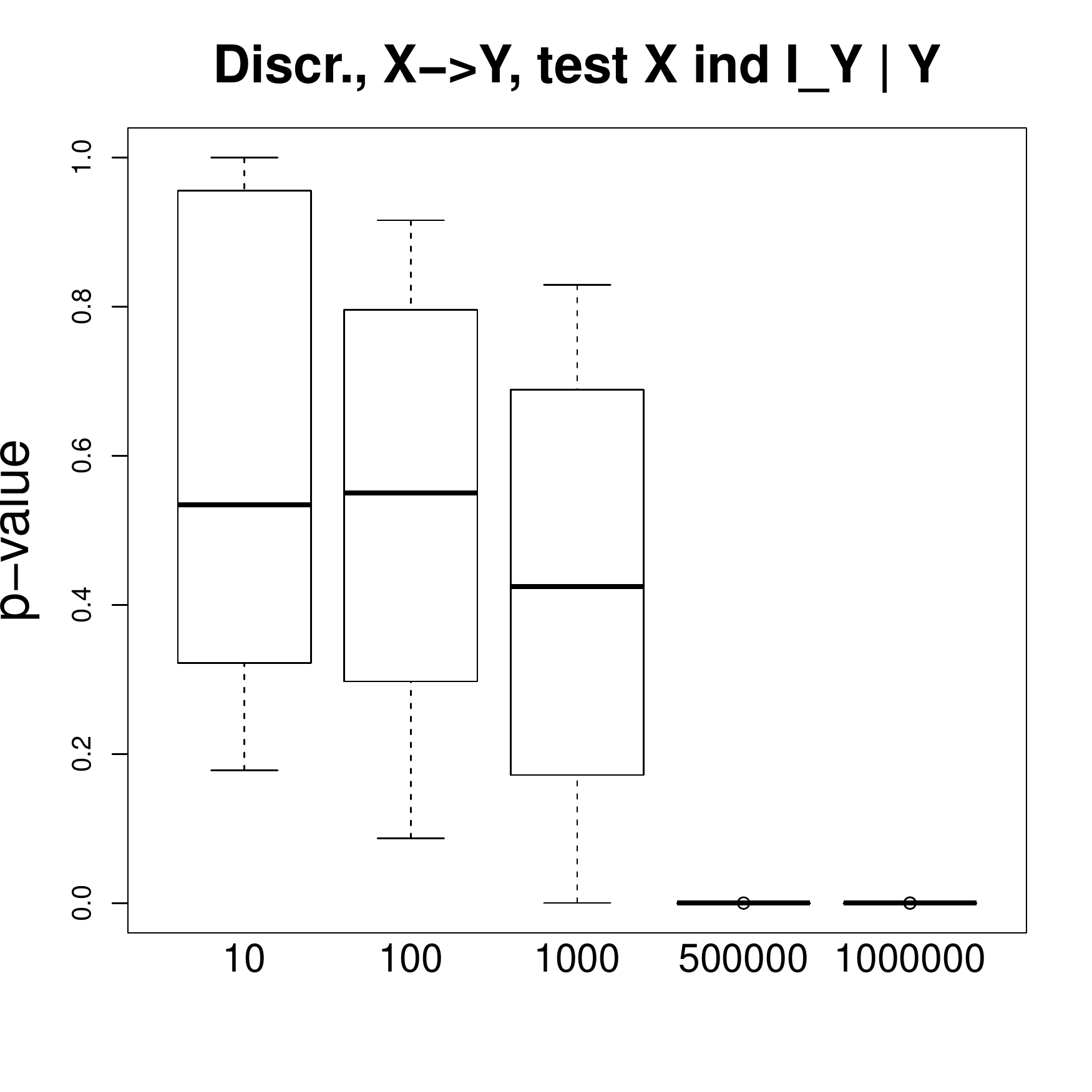}\includegraphics[scale=0.2]{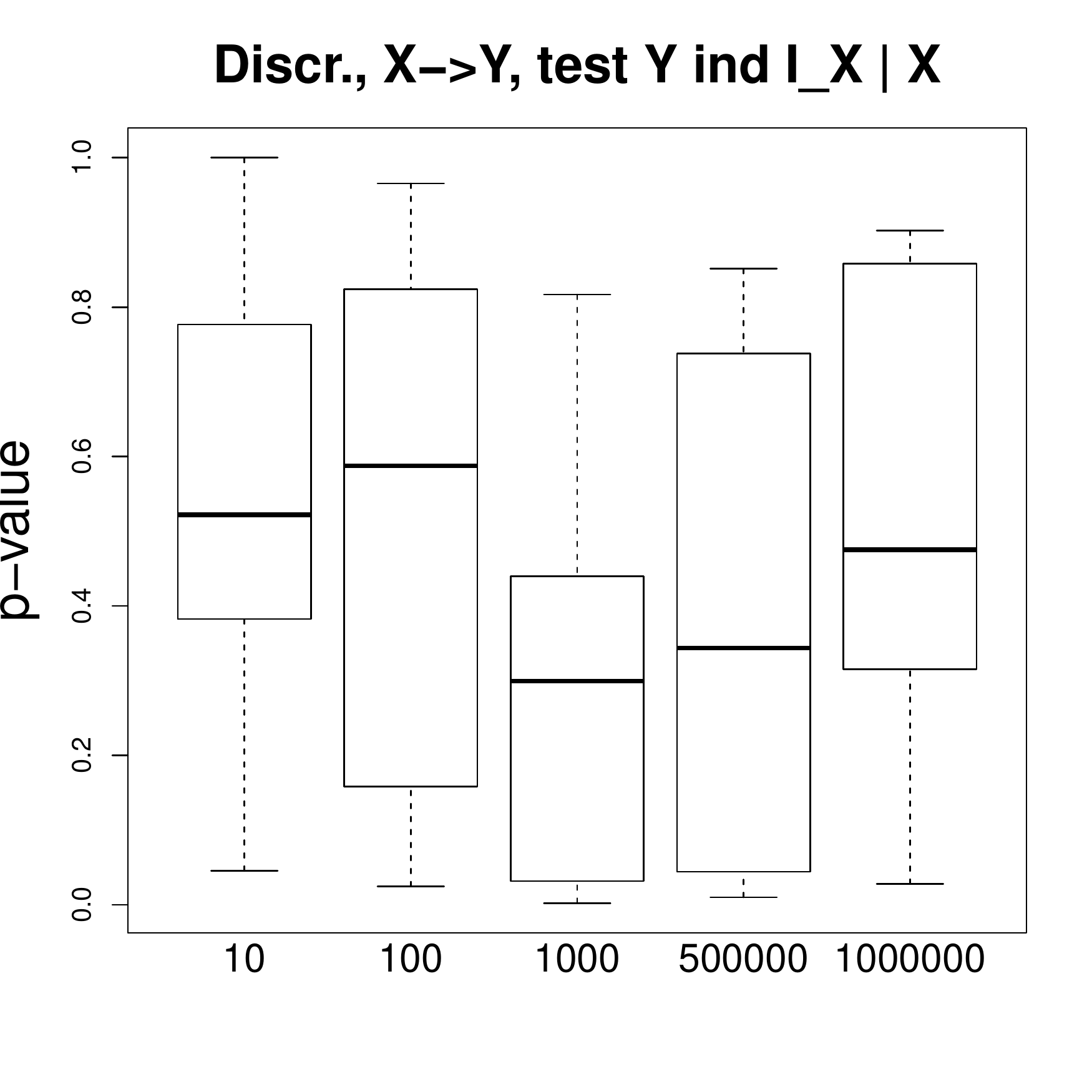}\includegraphics[scale=0.2]{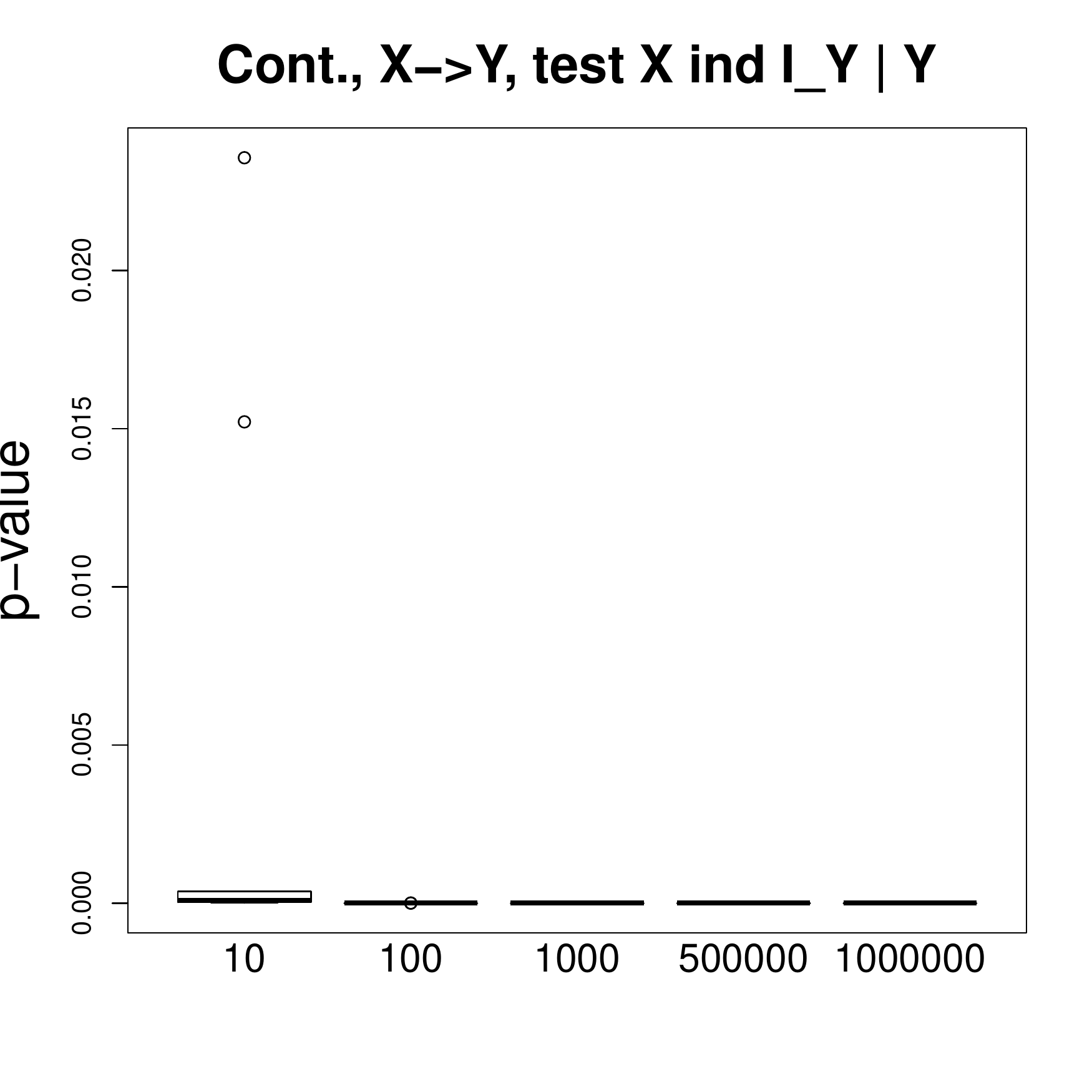}\includegraphics[scale=0.2]{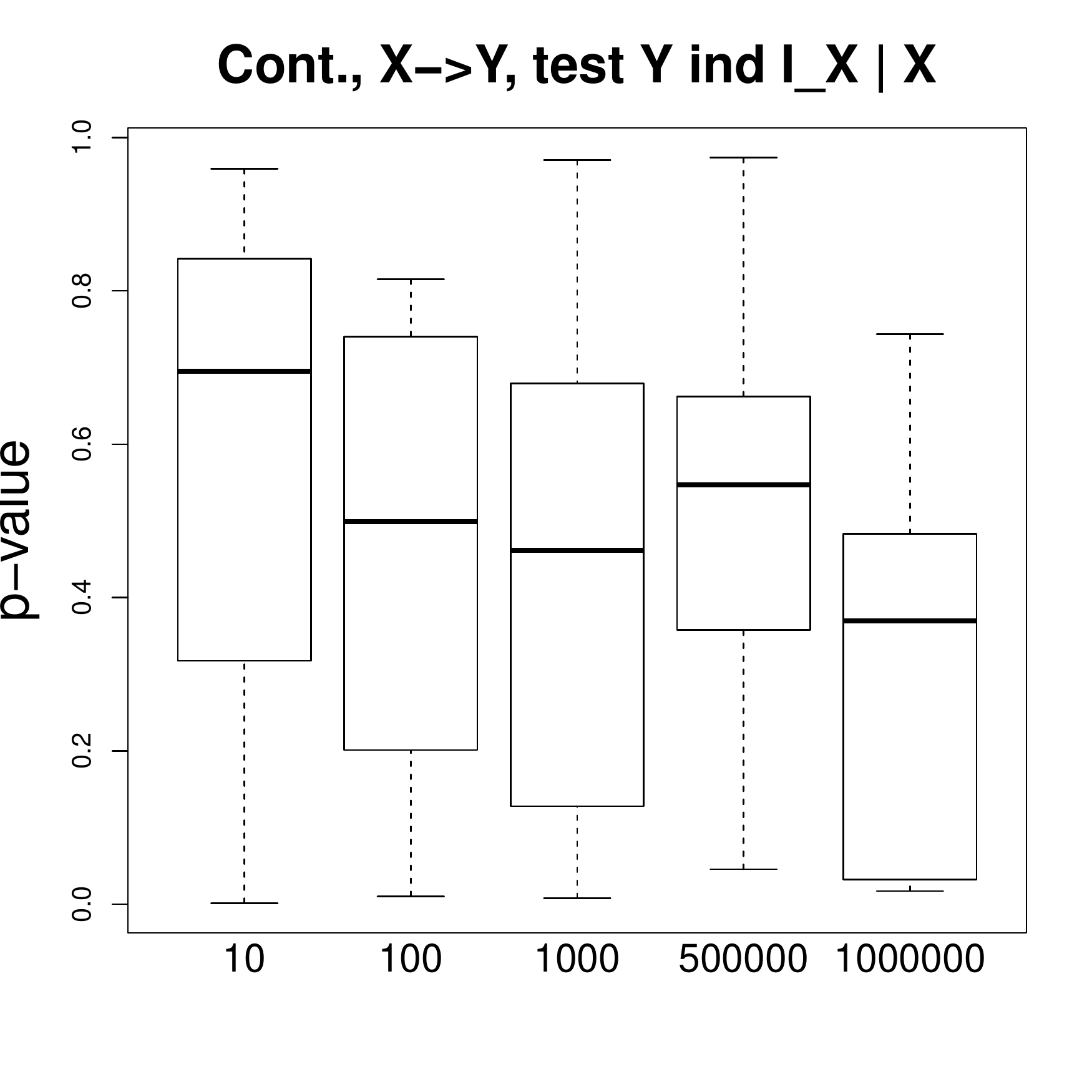}  
\caption{Simulated data. Ground truth: $X \rightarrow Y$. Two plots on the left: discrete data; two plots of the right: continuous data. The p-values of an asymptotic mutual information test (for the discrete case) and an exact t-test for Pearsons's correlation (the continuous case) as a function of the number of observations (x-axis).}
\label{fig:simulated1}
\end{figure*}
\begin{figure*}
\includegraphics[scale=0.2]{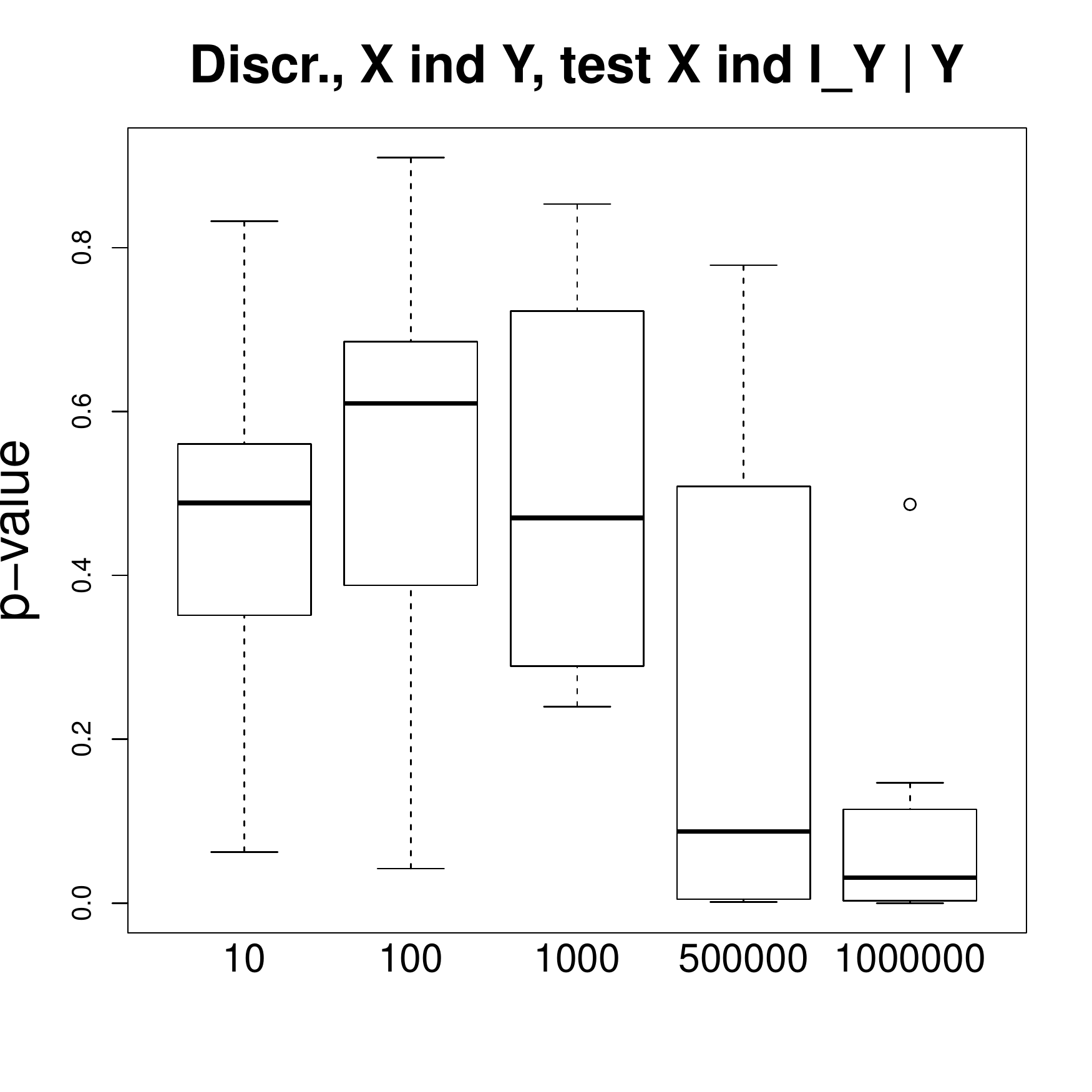}\includegraphics[scale=0.2]{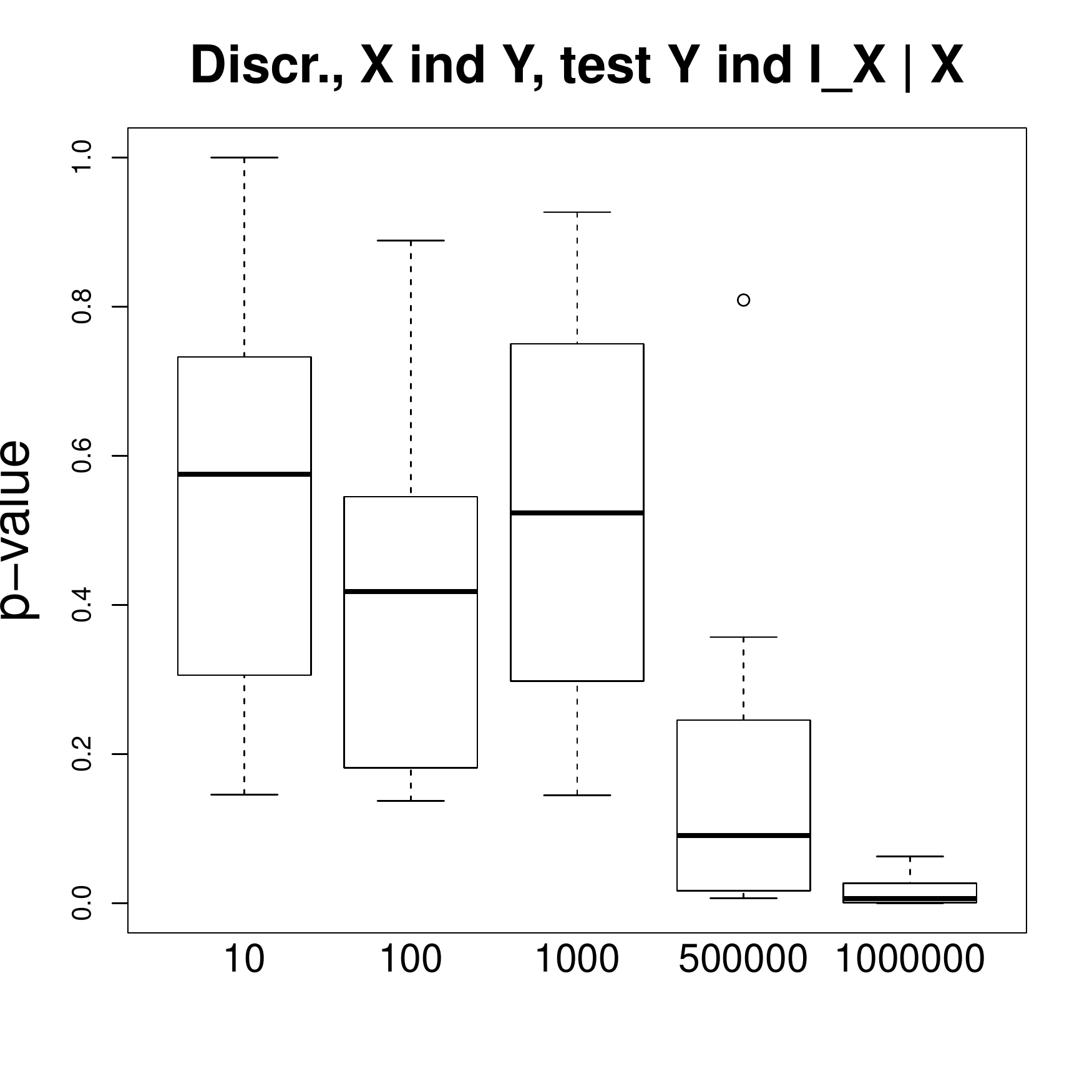}\includegraphics[scale=0.2]{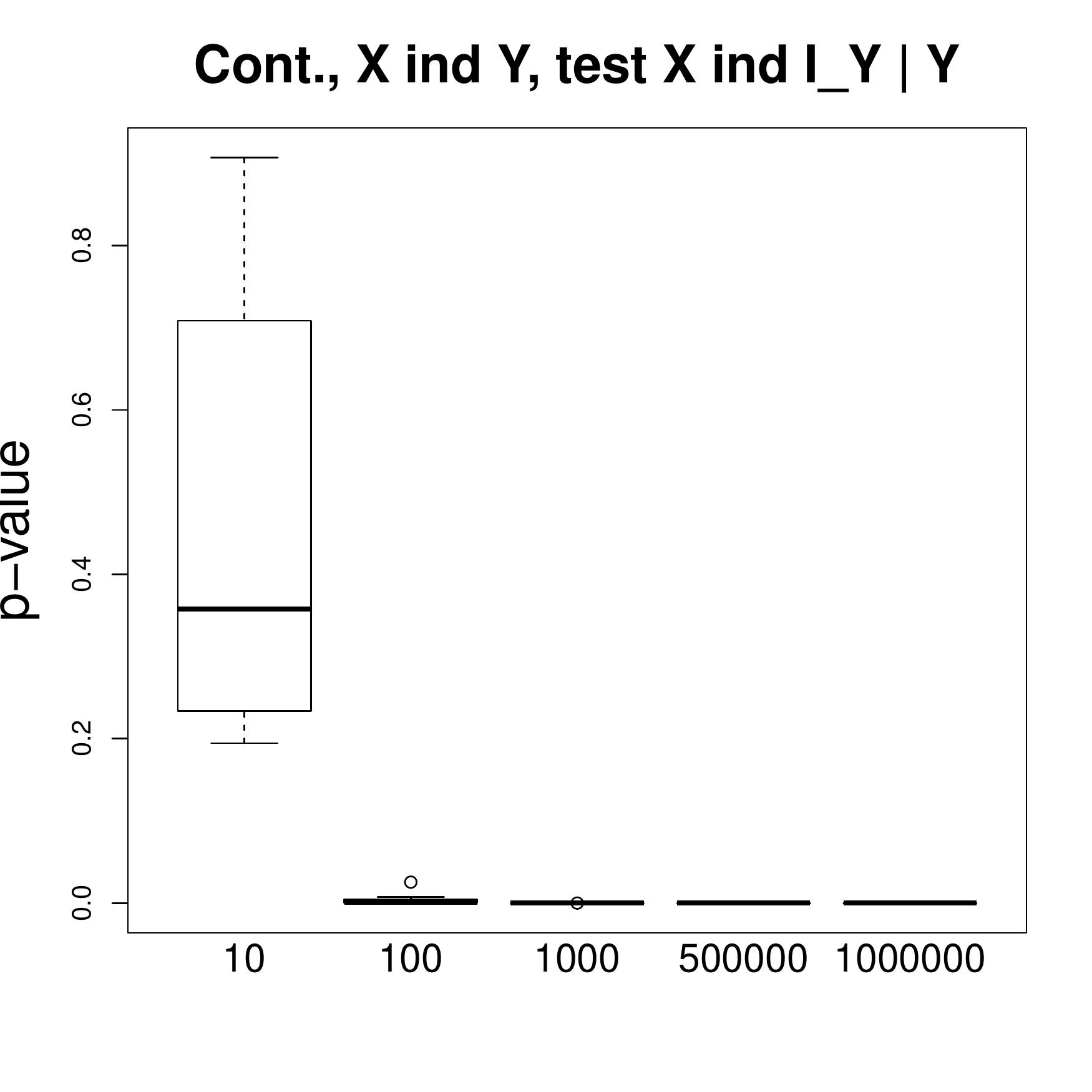}\includegraphics[scale=0.2]{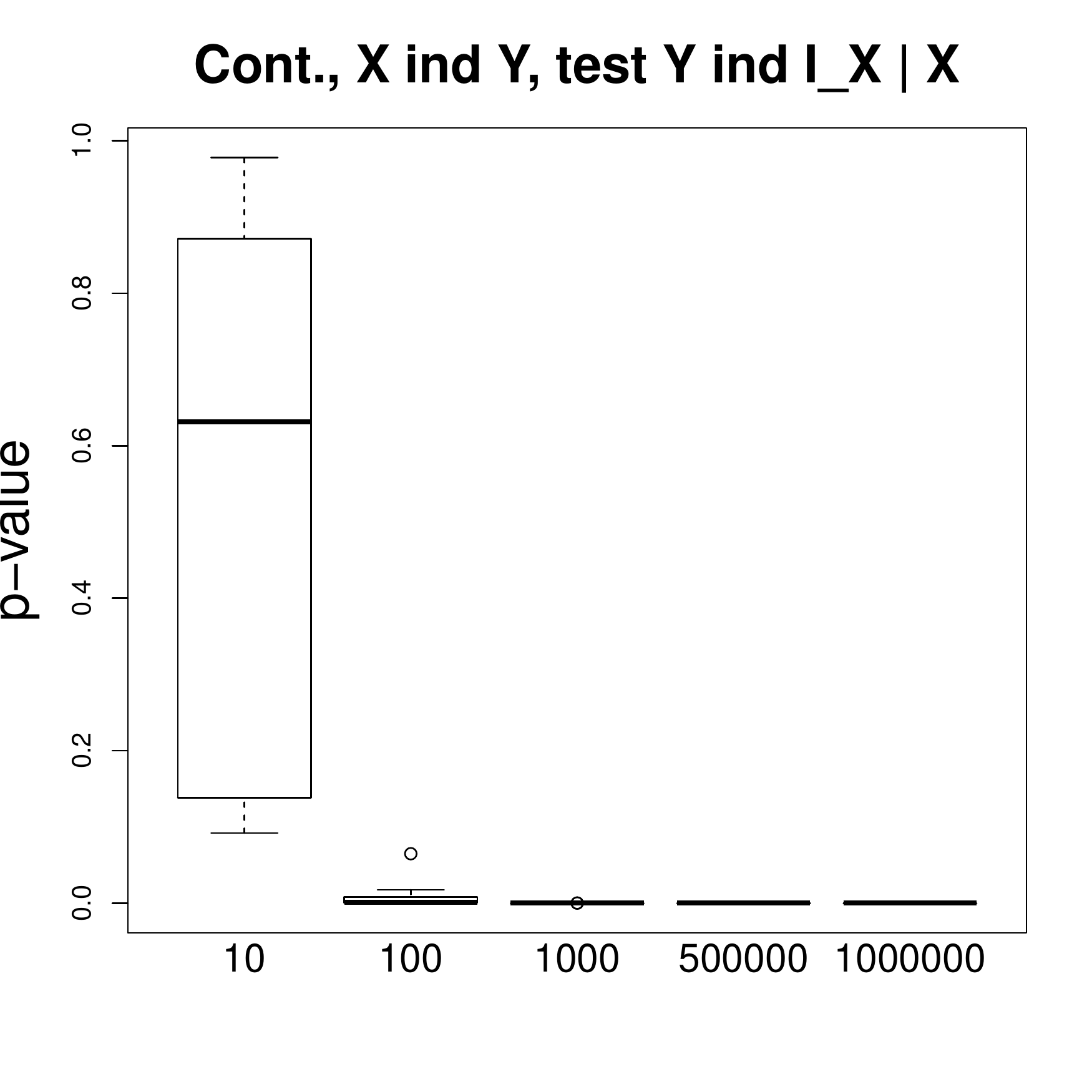}  
\caption{Simulated data. Ground truth: $X \CI Y | U$. On the left: two plots for the discrete setting; on the right: two plots for the continuous setting. The p-values as a function of the number of observations (x-axis). }
\label{fig:simulated2}
\end{figure*}
\begin{wrapfigure}{R}{0.5\textwidth}
\includegraphics[scale=0.2]{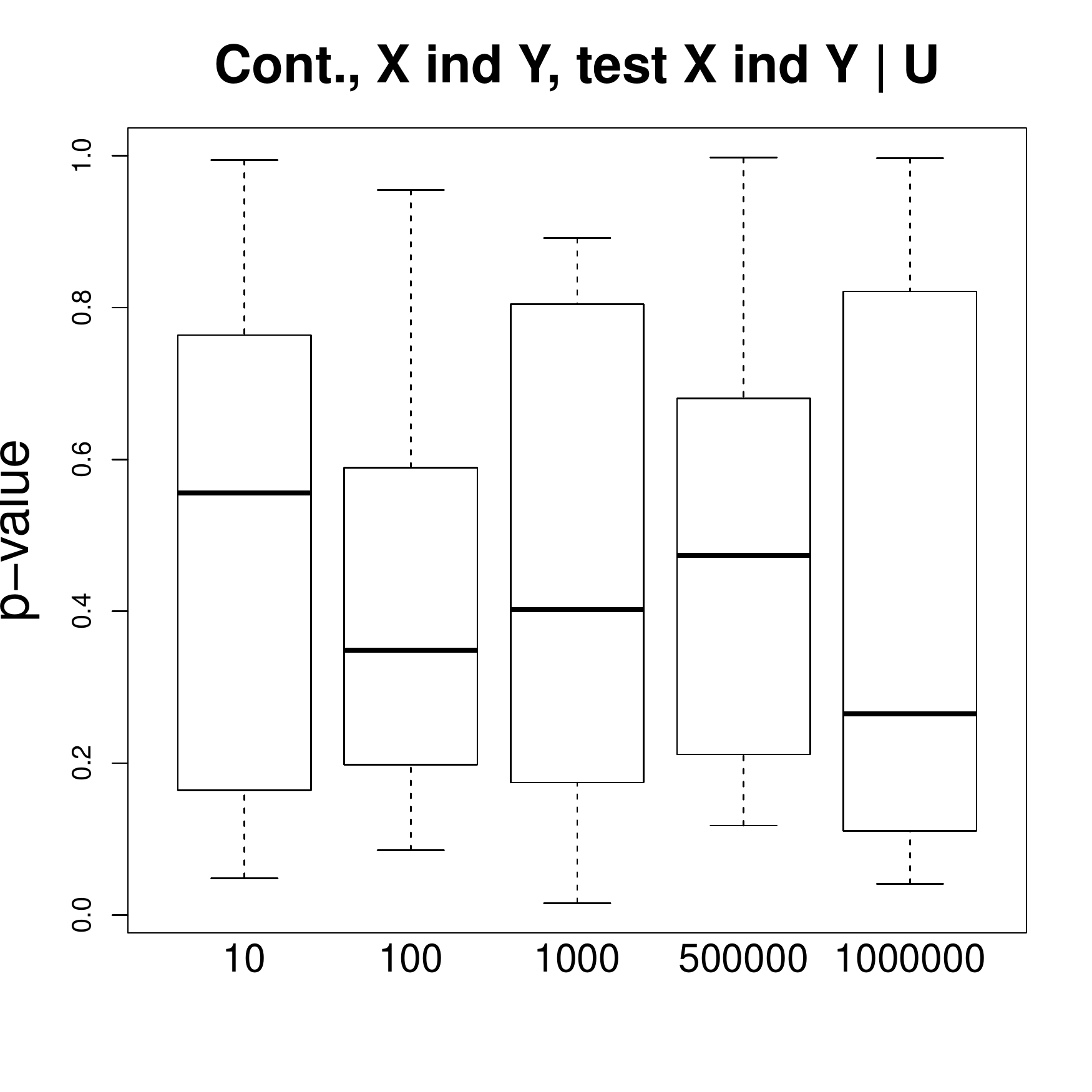}\includegraphics[scale=0.2]{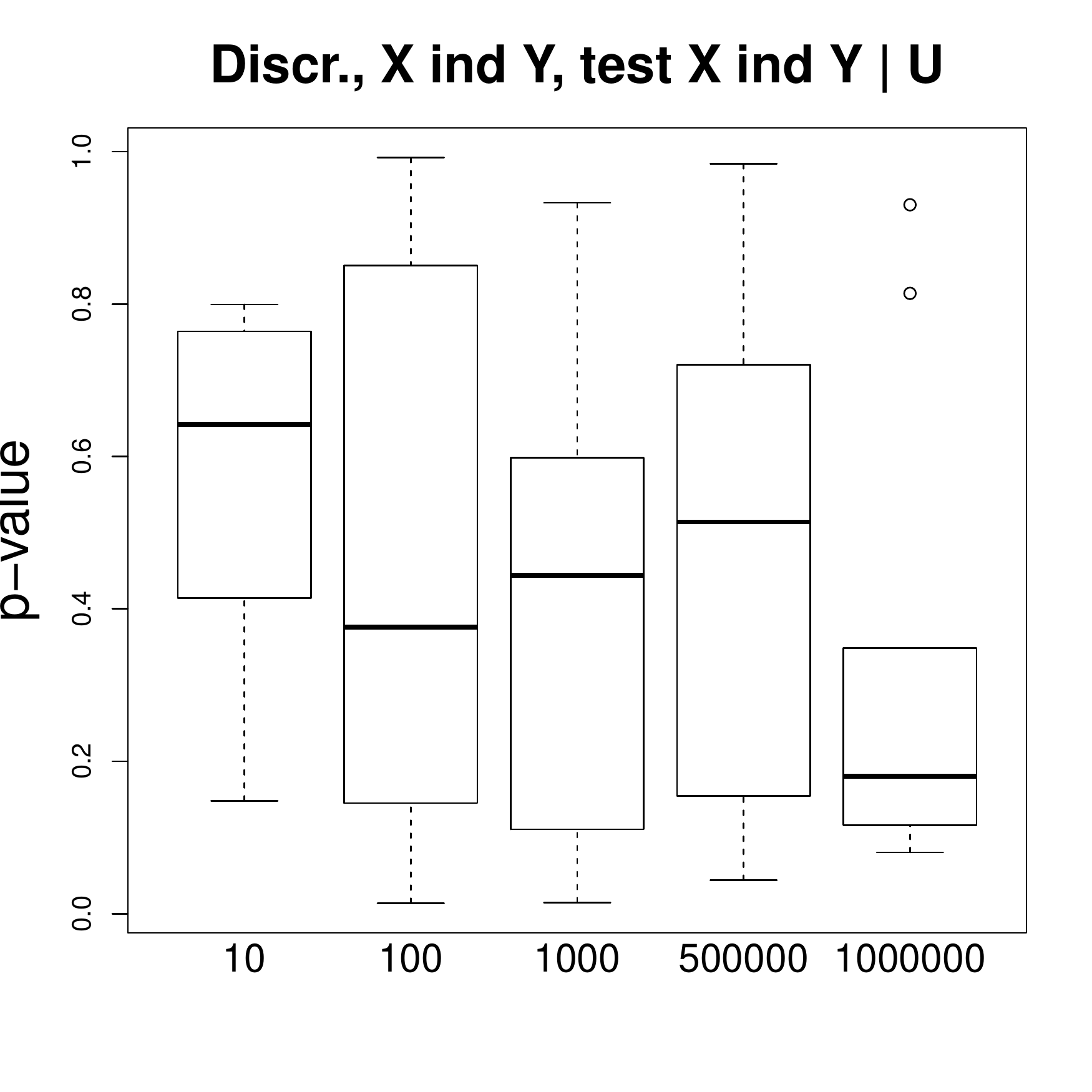} 
\caption{Simulated data. Ground truth: $X \CI Y | U$. The results of the conditional independence tests for $X \CI Y | U$ for continuous (on the left) and discrete (on the right) data. On the x-axis: the number of observations.}
\label{fig:simulated3}
\end{wrapfigure}

Figure~\ref{fig:simulated2} shows our results for the scenario $X \leftarrow U \rightarrow Y$. For the discrete and continuous experiments, we test whether $Y \CI I_X | X$  and whether $X \CI I_Y | Y$. We see that the variables are not independent. 
On Figure~\ref{fig:simulated3}, we demonstrate the p-values of the conditional independence test $Y \CI X | U$, what is more a sanity check, and we observe that in this case where the ground truth is $X \leftarrow U \rightarrow Y$, the p-values are far from 0 for both continuous and discrete scenarios.

\subsection{Cause-Effect Pairs}
We have tested the proposed algorithm on the benchmark collection of the cause-effect pairs, obtained from \url{http://webdav.tuebingen.mpg.de/cause-effect}, version 1.0.  The data set contains 100 pairs from different domains, and the ground truth is provided. The goal is to infer which variable is the cause and which is the effect. 

The pairs 52 -- 55, 70 -- 71,  and 81 -- 83 are excluded from the analysis, since they are multivariate problems. Note that each pair has an associated weight, provided with the data set,  since several cause-effect pairs can come from the same scientific problem, and in a number of publications reporting results on this data set, the accuracy is a weighted average. 
We apply the proposed method, described in Section~\ref{sec:ourmethod}, to infer causality on the benchmark data.  On Figure~\ref{fig:accPairs1}, we show the standard (unweighted) accuracy and the weighted accuracy where the weights for each observation pair are given in the data set. To increase the stability and accuracy, we also propose a scenario where we split the data into k-folds, carry out causal inference on each fold separately, and take an ensemble decision on the causal direction. The accuracy for such an ensemble approach is also shown on Figure~\ref{fig:accPairs1} for weighted accuracy using the ensemble method and the unweighted ensemble result. The number of folds in our experiments is 10. 
In our experiments, we show both the weighted and unweighted accuracy.  Speaking about the state-of-the-art results on the Cause-Effect pairs, it was reported that Origo~\cite{Budhathoki16} achieves 58\% accuracy, and the Additive Noise models (ANM)~\cite{Peters14} reach $72 \pm 6 \%$. Figure~\ref{fig:accPairs1} illustrates that the proposed method outperforms the state-of-the-art algorithms: the weighted accuracy which is compared to the above mentioned methods is 83.2\%. Note that the ensemble method reduced the variance significantly.

\begin{figure}
\centering
\includegraphics[scale=0.28]{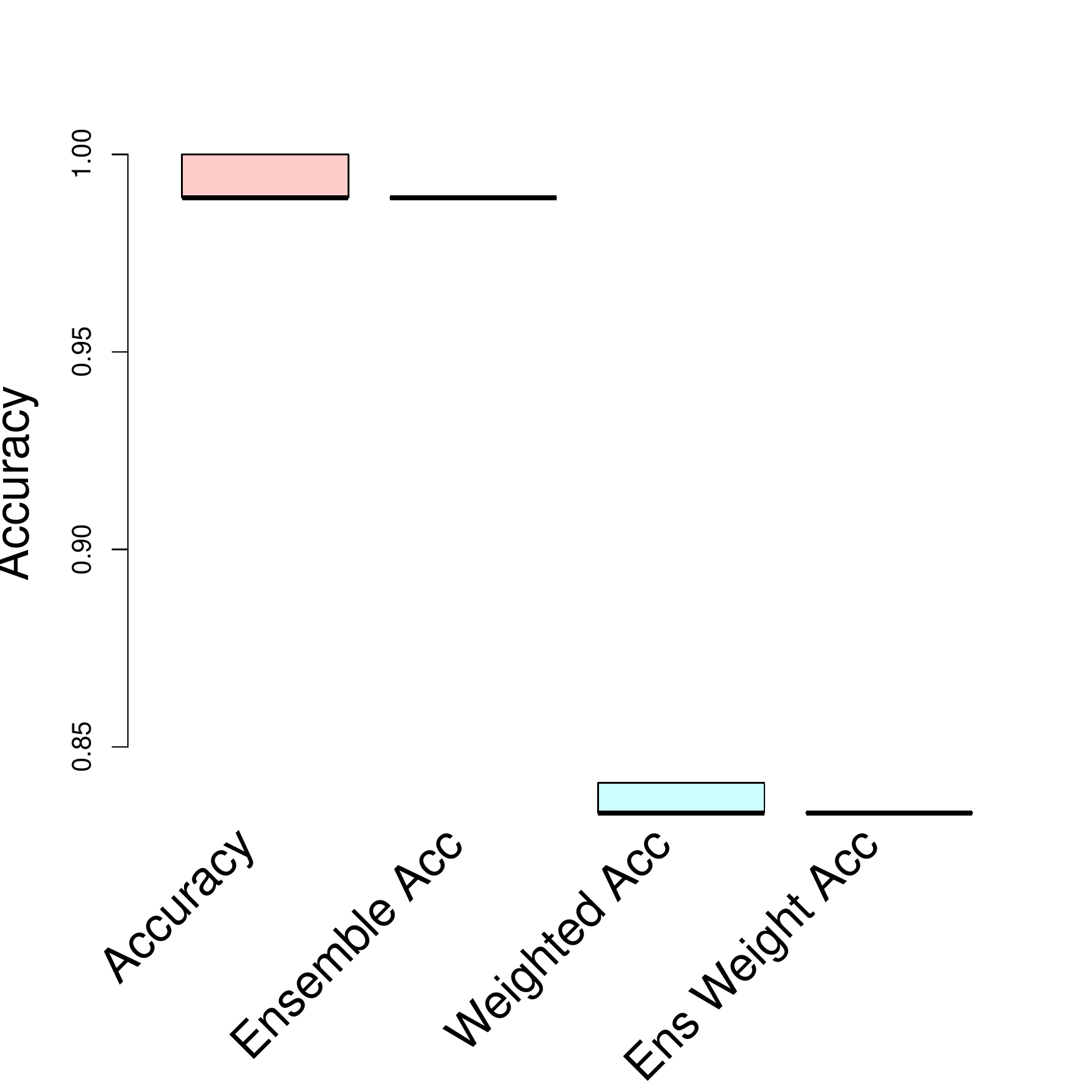}\includegraphics[scale=0.3]{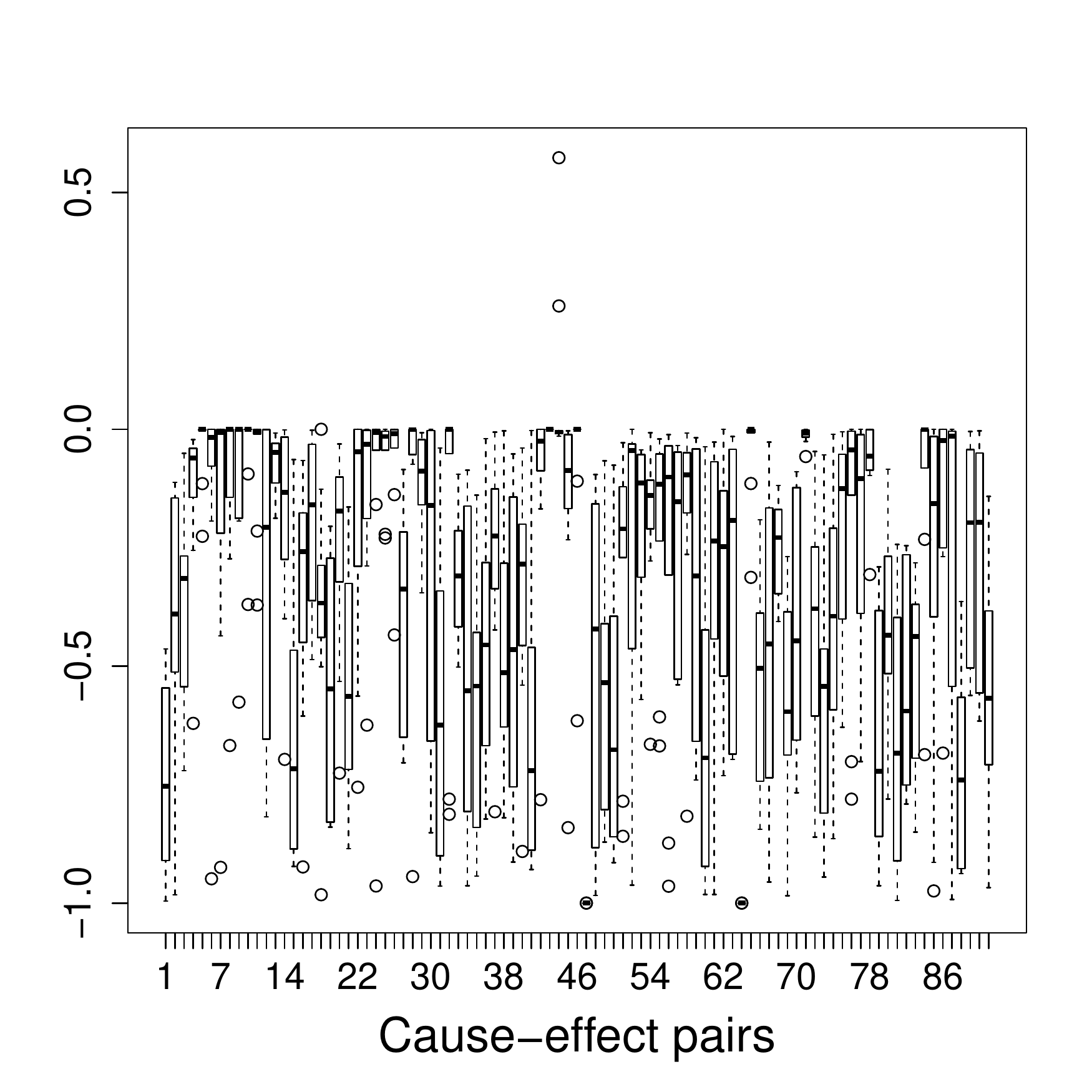}
\caption{On the left: accuracy on the Cause-Effect Benchmark. On the right: the difference between the test statistics $X \CI I_Y | Y$ and $Y \CI I_X | X$.}
\label{fig:accPairs1}
\end{figure}
Figure~\ref{fig:accPairs1} shows our results for the case where the number of clusters, i.e., modalities of the hidden instrumental variables is set to 15 both for $I_X$ and $I_Y$. We tested different numbers of clusters for the instrumental variables construction (see Section~\ref{sec:instrumental} for details). For the current task, we did not notice any important impact on the result, however, taking extremely small (2-3) and big (70 - 100) numbers of clusters degrades the performance.

What is central and what is interesting to look at, are the p-values of the conditional independence tests (here the exact t-test for Pearson's correlation from \texttt{bnlearn} R package) $X \CI I_Y | Y$ and $Y \CI I_X | X$. On Figure~\ref{fig:accPairs1} on the right we show their difference. If the p-values of the test $X \CI I_Y | Y$ are small (that is $X$ and $I_Y$ are not independent given $Y$) and the results of  $Y \CI I_X | X$ are relatively big (or bigger than ones of $X \CI I_Y | Y$) telling that $Y$ and $I_X$ are independent given $X$, then the plotted difference is negative. It is exactly what is observed for almost all cause-effect pairs.

\section{Conclusions}

We challenged to bring together two principle research avenues on causal inference, i.e. causal inference using conditional independence and methods based on the postulate of independence of cause and mechanism. The causal inference methods based on the independence of cause and mechanism, or, to be precise, between the prior and the mechanism, and that consider probability distributions as random variables, appear more and more often in the machine learning community.  In this contribution, we provide some theoretical foundations for this family of algorithms. Our main message is that the role of the hidden instrumental variables can not be neglected.

We propose an algorithm to estimate the latent instrumental variables efficiently. We also introduce a simple (and symmetric) algorithm to perform causal inference for the case of only two observed variables, where the corresponding instrumental variables are approximated. Our original approach is simple to implement, since it is based on a clustering algorithm (we used the k-means clustering, but any other clustering method can be tested) and on conditional independence tests. It can be applied to both discrete and continuous data, and we have shown that it is extremely competitive compared to the state-of-the-art methods on a real benchmark, where a cluster assumption holds.

Currently, we consider an extension of the proposed algorithm to more complex graphs, and potentially huge applications such as modelling gene interactions. Another avenue of research are novel metrics to measure the conditional independence of variables.

\bibliographystyle{plain}
\bibliography{causality}

\end{document}